\documentclass[runningheads]{llncs}

\usepackage[utf8]{inputenc}
\DeclareUnicodeCharacter{FB01}{fi}

\usepackage[linesnumbered,ruled,vlined]{algorithm2e}
\usepackage{amsmath}
\usepackage{amssymb}
\usepackage{booktabs}
\usepackage{cite}
\usepackage{enumitem}
\usepackage{graphicx}
\usepackage{url}
\usepackage{wrapfig}
\usepackage[dvipsnames]{xcolor}
\usepackage{xspace}
\usepackage{fontenc}
\usepackage{mathtools}
\usepackage{listings}
\usepackage[frozencache]{minted}
\usepackage{syntax}

\usepackage{tikz}

\newcommand{\set}[1]{\{ #1 \}}
\newcommand{\aset}[1]{\langle #1 \rangle}

\newcommand{\eg}{e.g.,\xspace}

\newcommand{\figref}[1]{Fig.~\ref{#1}}
\newcommand{\tabref}[1]{Table~\ref{#1}}
\newcommand{\secref}[1]{Section~\ref{#1}}

\newcommand{\ignore}[1]{}

\newcommand{\DNN}{\mathcal{N}}

\newcommand{\planet}{{\sl Planet}\xspace}
\newcommand{\reluplex}{{\sl Reluplex}\xspace}
\newcommand{\eran}{{\sl ERAN}\xspace}
\newcommand{\neurify}{{\sl Neurify}\xspace}
\newcommand{\mipverify}{{\sl MIPVerify}\xspace}
\newcommand{\marabou}{{\sl Marabou}\xspace}
\newcommand{\nnenum}{{\sl nnenum}\xspace}
\newcommand{\bab}{{\sl BaB}\xspace}
\newcommand{\babsb}{{\sl BaBSB}\xspace}
\newcommand{\deepzono}{{\sl DeepZono}\xspace}
\newcommand{\deeppoly}{{\sl DeepPoly}\xspace}
\newcommand{\refinezono}{{\sl RefineZono}\xspace}
\newcommand{\refinepoly}{{\sl RefinePoly}\xspace}
\newcommand{\verinet}{{\sl VeriNet}\xspace}

\renewcommand*{\checkmark}[1][]{\tikz[x=1em, y=1em]\fill[#1] (0,.35) -- (.25,0) -- (1,.7) -- (.25,.15) -- cycle;}
\newcommand{\filledcirc}[1]{%
	\begin{tikzpicture}
		\draw (0,0) circle (1ex);
		\ifthenelse{\equal{#1}{left}}{\fill (0,1ex) arc (90:270:1ex) -- (0,0) -- cycle;}{
			\ifthenelse{\equal{#1}{right}}{\fill (0,-1ex) arc (90:270:-1ex) -- (0,0) -- cycle;}{
				\ifthenelse{\equal{#1}{rightgray}}{\fill[gray] (0,-1ex) arc (90:270:-1ex) -- (0,0) -- cycle;}{
					\ifthenelse{\equal{#1}{top}}{\fill (1ex,0) arc (0:180:1ex) -- (0,0) -- cycle;}{
						\ifthenelse{\equal{#1}{bottom}}{\fill (-1ex,0) arc (0:180:-1ex) -- (0,0) -- cycle;}{
							\ifthenelse{\equal{#1}{full}}{\fill (1ex,0) arc (0:360:1ex) -- (0,0) -- cycle;}{
								\ifthenelse{\equal{#1}{fullgray}}{
									\fill (0,1ex) arc (90:270:1ex) -- (0,0) -- cycle;
									\fill[gray] (0,-1ex) arc (90:270:-1ex) -- (0,0) -- cycle;
									}{
									\ifthenelse{\equal{#1}{empty}}{}{
										\ifthenelse{\equal{#1}{}}{}{
											\errmessage{Unknown option: #1}}}}}}}}}}
									\end{tikzpicture}%
									}

\input{pygmentize}

\newcommand{\approach}{\textsc{DNNV}\xspace}
\newcommand{\tool}{\textsc{DNNV}\xspace}
\newcommand{\onnx}{\textsc{ONNX}\xspace}
\newcommand{\dsl}{\textsc{DNNP}\xspace}

\begin{document}

\title{DNNV: A Framework for Deep Neural Network Verification}

\author{David Shriver\orcidID{0000-0003-0208-6517} \and
	Sebastian Elbaum\orcidID{0000-0001-9592-1352} \and
	Matthew B. Dwyer\orcidID{0000-0002-1937-1544}}

\authorrunning{D. Shriver et al.}

\institute{University of Virginia, Charlottesville, VA, USA
	\email{\{dls2fc,selbaum,matthewbdwyer\}@virginia.edu}}

\maketitle

\begin{abstract}
	Despite the large number of sophisticated deep neural network (DNN) verification algorithms, DNN verifier developers, users, and researchers still face several challenges.
	First, verifier developers must contend with the rapidly changing DNN field to support new DNN operations and property types.
	Second, verifier users have the burden of selecting a verifier input format to specify their problem.
	Due to the many input formats, this decision can greatly restrict the verifiers that a user may run.
	Finally, researchers face difficulties in re-using benchmarks to evaluate and compare verifiers, due to the large number of input formats required to run different verifiers.
	Existing benchmarks are rarely in formats supported by verifiers other than the one for which the benchmark was introduced.
	In this work we present DNNV, a framework for reducing the burden on DNN verifier researchers, developers, and users. \tool standardizes input and output formats, includes a simple yet expressive DSL for specifying DNN properties, and provides powerful simplification and reduction operations to  facilitate the application, development, and comparison of DNN verifiers.
	We show how DNNV  increases the support of verifiers for existing benchmarks from 30\% to 74\%.
\end{abstract}

\section{Introduction}
\label{sec:introduction}

Deep neural networks (DNN) are being applied increasingly in complex domains including
safety critical systems such as autonomous driving~\cite{codevilla-etal:ICRA:2018:driving,bojarski-etal:NIPS-DLS:2016:end-to-end-driving}.
For such applications, it is often necessary to obtain behavioral
guarantees about the safety of the system.
To address this need, researchers have been  exploring algorithms
for verifying that the behavior of a trained DNN meets some correctness
property.
In the past few years, more than 20 DNN verification algorithms have
been introduced~\cite{bastani-etal:NIPS:2016,boopathy-etal:AAAI:2019:cnncert,bunel-etal:NIPS:2018:plnn,dutta:NFM:2018:sherlock,dvijotham-etal:UAI:2018,ehlers:ATVA:2017:planet,gehr-etal:SP:2018:ai2,huang-etal:CAV:2017:dlv,raghunathan-etal:ICLR:2018,ruan-etal:IJCAI:2018:deepgo,singh-etal:NeurIPS:2019:refinepoly,singh-etal:NIPS:2018:deepzono,singh-etal:POPL:2019:deeppoly,singh-etal:ICLR:2019:refinezono,tjeng-etal:ICLR:2019:mipverify,wang-etal:NIPS:2018:neurify,wang-etal:USS:2018:reluval,weng-etal:ICML:2018,wong-kolter:ICML:2018,xiang-etal:NNLS:2018,zhang-etal:NeurIPS:2018:crown}, and this number continues to grow.
Unfortunately, this progress is hindered by several challenges.

First, DNN verifier developers must contend with a rapidly changing field that 
continually incorporates new DNN operations and property types. 
While supporting more properties and operations may increase the applicable scope of verifiers to real-world problems, it also increases a verifier's complexity.
For example, for a verifier such as \deeppoly, supporting additional operations requires non-trivial effort to define and prove correctness of new abstract transformers. 
For verifiers such as \reluplex or \neurify, supporting new property types requires implementing a mapping from those properties onto internal verifier structures.

Second, DNN verifier users carry the burden of re-writing property specifications  and transforming their models to match a chosen verifier's supported format. 
That burden is compounded by the diversity of input formats required by each verifier,  as illustrated in \tabref{tab:verifiers}. 
There is little overlap between input formats for verifiers (only 
\deepzono and \deeppoly or \bab and \babsb which are algorithmically similar), and even when using the same format (as in the case of the popular \onnx format) we find that the underlying operations supported are  different.
This makes it difficult and costly to run multiple verifiers on a given problem since the user must understand the requirements of each verifier and translate inputs to their formats.
While two new formats, VNNLIB~\cite{VNNLIB} and SOCRATES~\cite{SOCRATES}, have been introduced in an attempt to standardize DNN verifier input formats, their expressiveness is currently limited and they can require writing new conversion tools for networks, as we discuss at the end of \secref{subsec:inputformat:others}.

\begin{table}[t]
	\scriptsize
	\centering
	\caption{The network and property formats supported by each verifier. A * indicates that only a subset of the full input format specification is supported.}
	\label{tab:verifiers}
	\begin{tabular}{llll}
		\toprule
		\textbf{Verifier}                                     & \textbf{Network Format}       & \textbf{Property Format} & \textbf{Algorithmic Approach} \\
		\midrule
		\reluplex~\cite{katz-etal:CAV:2017:reluplex}          & \reluplex-NNET                & hard-coded               & Search                        \\
		\planet~\cite{ehlers:ATVA:2017:planet}                & RLV                           & RLV                      & Search                        \\
		\bab~\cite{bunel-etal:NIPS:2018:plnn}                 & RLV                           & RLV                      & Search                        \\
		\babsb~\cite{bunel-etal:NIPS:2018:plnn}               & RLV                           & RLV                      & Search                        \\
		\mipverify~\cite{tjeng-etal:ICLR:2019:mipverify}      & MIPVerify Julia API           & MIPVerify Julia API      & Optimization                  \\
		\neurify~\cite{wang-etal:NIPS:2018:neurify}           & \neurify-NNET                 & hard-coded               & Search-Optimization           \\
		\deepzono~\cite{singh-etal:NIPS:2018:deepzono}        & ONNX*, \eran-PYT, \eran-TF & \eran Python API         & Reachability                  \\
		\deeppoly~\cite{singh-etal:POPL:2019:deeppoly}        & ONNX*, \eran-PYT, \eran-TF & \eran Python API         & Reachability                  \\
		\refinezono~\cite{singh-etal:ICLR:2019:refinezono}    & ONNX*, \eran-PYT, \eran-TF & \eran Python API         & Reachability                  \\
		\refinepoly~\cite{singh-etal:NeurIPS:2019:refinepoly} & ONNX*, \eran-PYT, \eran-TF & \eran Python API         & Reachability                  \\
		\marabou~\cite{katz-etal:cav:2019:marabou}            & \reluplex-NNET or ONNX*       & \marabou Python API      & Search                        \\
		\nnenum~\cite{bak-etal:cav:2020:nnenum}               & ONNX*                         & \nnenum Python API       & Search-Reachability           \\
		\verinet~\cite{henriksen-lomuscio:ECAI:2020:verinet}  & ONNX* or \neurify-NNET        & \verinet Python API      & Search-Optimization           \\
		\bottomrule
	\end{tabular}
\end{table}

Finally, DNN verifier researchers face challenges in re-using benchmarks to evaluate and compare verifiers.
Most benchmarks exist in the format of the verifier for which they were introduced, and running other verifiers on that benchmark requires writing custom tooling to translate the benchmark to other formats, or writing new input parsers for verifiers to support the given benchmark format.
For example, the ACAS Xu benchmark (described in \secref{sec:study}), was originally specified with networks in \reluplex-NNET format, and properties hard-coded into the verifier.
The benchmark was converted, for example, into RLV format for \bab and \babsb, as well as into \onnx with hard-coded properties for \refinezono.
Other benchmarks, such as the DAVE benchmark used by \neurify, has networks specified in \neurify-NNET, and properties hard-coded into the verifier. Due to its format, this potentially great benchmark has not been used by other verifiers.

\emph{
	We introduce a framework, \tool, to reduce the burden on
	verifier researchers, developers, and users.
}
\tool helps to create and run more re-usable verification benchmarks by standardizing a network and property format, and
it increases the applicability of a verifier to richer properties and real-world benchmarks by performing property reductions and simplifying DNN structures.

As shown in Fig. \ref{fig:diagram},
\tool takes as input a network in the common  
\onnx input format, a property written in an expressive domain-specific language DNNP, and the name of a target verifier.
Using the framework and plugins for the target verifier,
\tool transforms the problem by simplifying the network and reducing the property
to enable the application of verifiers that otherwise would be unable to run.
\tool then translates the network and property to the input format of the desired verifier, runs that verifier on the transformed problem, and returns the results in a standardized format.

\begin{figure}[t]
	\centering
	\includegraphics[width=\linewidth]{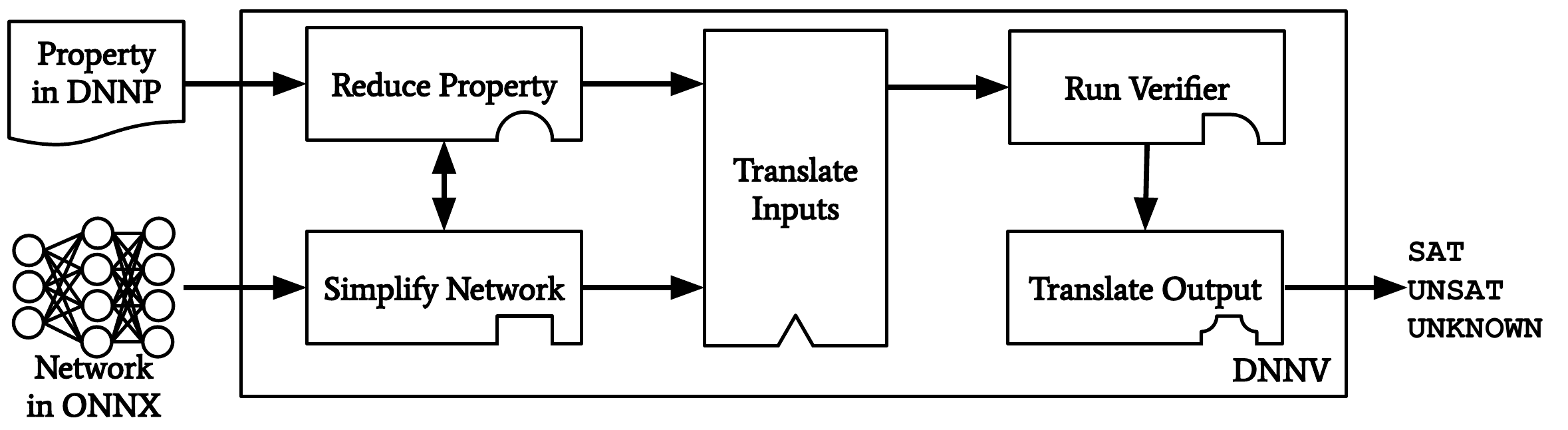}
	\caption{
		\tool Architecture
	}
	\label{fig:diagram}
\end{figure}

The primary contributions of this work are:
(1) the \tool framework  to reduce the burden on DNN verifier researchers, developers, and users; \tool includes a simple yet expressive DSL for specifying DNN properties, and powerful simplification and reduction operations to increase verifiers' scope of applicability,
(2) an open source tool implementing \tool\footnote{\url{https://github.com/dlshriver/DNNV}}, with support for 13 verifiers, and extensive documentation, and
(3) an evaluation demonstrating the cost-effectiveness of \tool to increase the scope of applicability of verifiers.

\section{Background}
\label{sec:background}

A deep neural network $\DNN$ encodes an approximation of a target function $f: \mathbb{R}^{n} \rightarrow \mathbb{R}^{m}$.
A DNN can be represented as a directed graph $G_\DNN = \aset{V_\DNN, E_\DNN}$, where nodes, $v \in V_\DNN$, represent operations and edges, $e \in E_\DNN$, represent input arguments to operations.
A node without any incoming edges is an input to the DNN.
The output of a DNN can be computed by looping over nodes in topological order and computing the value of the node given its inputs.
The literature on machine learning has developed a broad range of
rich operation types and explored the benefits of different combinations of operations
in realizing accurate approximations of different target functions, e.g., \cite{Goodfellow-et-al-2016}.

Given a DNN, $\DNN: \mathbb{R}^{n} \rightarrow \mathbb{R}^{m}$,
a property, $\phi(\DNN)$, defines a set of constraints over the inputs, $\phi_{\mathcal{X}}$ -- the pre-condition, and a set of constraints over the 
outputs, $\phi_{\mathcal{Y}}$ -- the post-condition.
Verification of $\phi(\DNN)$ seeks to prove or falsify:
$\forall{x\in \mathbb{R}^{n}}: \phi_{\mathcal{X}}(x) \rightarrow \phi_{\mathcal{Y}}(\DNN(x))$.

A widely studied class of properties is \textit{robustness}, which originated with the study of adversarial examples~\cite{szegedy-etal:iclr:2014:lbfgs,yuan-etal:TNNLS:2019:aesurvey}.
These properties   specify that inputs from a specific region of the input space must all produce the same output class.
Detecting violations of robustness properties has been widely studied, and they are a
common type of property for evaluating verifiers~\cite{wang-etal:NIPS:2018:neurify,singh-etal:NIPS:2018:deepzono,singh-etal:POPL:2019:deeppoly,ehlers:ATVA:2017:planet,tjeng-etal:ICLR:2019:mipverify}.
Another common class of properties is \textit{reachability}, which define the post-condition using constraints over output values.
Reachability properties specify that inputs from a given region of the input space must produce outputs within a given region of the output space. 
Such properties have been used to evaluate several DNN verifiers~\cite{wang-etal:NIPS:2018:neurify,katz-etal:CAV:2017:reluplex,katz-etal:cav:2019:marabou}.

A recent survey on DNN verification~\cite{liu-etal:arxiv:2019:neuralverification} classifies these approaches based on their type: reachability, optimization, or search, or a combination of these.
Reachability-based methods compute a representation of the reachable set of outputs from an encoding of the set of inputs that satisfy the pre-condition.
The computed output set is often an over-approximation of the true reachable output region.
The precision of the computed output region depends on the symbolic representation used, \eg hyper-rectangles, zonotopes, polyhedra.
Reachability-based methods include~\cite{gehr-etal:SP:2018:ai2,ruan-etal:IJCAI:2018:deepgo,singh-etal:NeurIPS:2019:refinepoly,singh-etal:NIPS:2018:deepzono,singh-etal:POPL:2019:deeppoly,singh-etal:ICLR:2019:refinezono,xiang-etal:NNLS:2018}.
Optimization-based methods formulate property violations as a threshold for an objective function and use optimization algorithms to attempt to satisfy that threshold.
Optimization-based methods include~\cite{bastani-etal:NIPS:2016,dvijotham-etal:UAI:2018,raghunathan-etal:ICLR:2018,tjeng-etal:ICLR:2019:mipverify,wong-kolter:ICML:2018}.
Search-based methods explore regions of the input space where they then
formulate reachability or optimization sub-problems.
Search-based methods include~\cite{wang-etal:USS:2018:reluval,weng-etal:ICML:2018,huang-etal:CAV:2017:dlv,katz-etal:CAV:2017:reluplex,ehlers:ATVA:2017:planet,bunel-etal:NIPS:2018:plnn}.

\section{\tool Overview}
\label{sec:overview}

\tool remedies several key challenges faced by the DNN verification community.
A general overview of \tool is shown in \figref{fig:diagram}.
\tool takes in a property and network in a standard format, simplifies the network, reduces the property, translates the network and property to the input format of the verifier, runs the verifier, and translates its output.
Each of these components can be customized by verifier specific plugins.
We explain these components in more detail below.

\subsection{Input Formats}
\begin{wraptable}[8]{r}{0.45\textwidth}
	\vspace{-3.5em}
	\footnotesize
	\centering
	\caption{The number of ONNX operations supported by each verifier.}
	\label{tab:onnxsupport}
	\begin{tabular}{lr}
		\textbf{Verifier} & \textbf{\# ONNX Operations} \\
		\toprule
		DNNV              & 31                          \\
		ERAN              & 22                          \\
		nnenum            & 15                          \\
		marabou           & 12                          \\
		VeriNet           & 12                          \\
		\bottomrule
	\end{tabular}
\end{wraptable}
As shown in Table \ref{tab:verifiers}, existing verifiers do not support a consistent, common input format for networks and properties.
\tool standardizes the input and output formats to aid the community in creating and running verification benchmarks.

\textbf{ONNX}\label{subsec:inputformat:onnx} For specifying general deep neural network architectures,
we choose the open source DNN format ONNX  \cite{onnx}.
ONNX  can represent real-world networks, is supported by many common frameworks (e.g., PyTorch, MXNet) and conversion tools are available for other frameworks (e.g., TensorFlow, Keras).
Our current implementation supports a subset of the ONNX specification
that subsumes the subsets of ONNX implemented by the supported verifiers.
\tabref{tab:onnxsupport} shows the number of ONNX operations supported by each of the verifiers included in DNNV.
DNNV supports 40\% more operations than the verifier with the next highest support.
The ONNX subset supported by DNNV is sufficient for almost all existing verification benchmarks, as well as many real-world networks including VGG16 and ResNet34.

\begin{wrapfigure}[19]{r}{0.55\textwidth}
	\vspace{-3em}
	{
		\scriptsize
		\input{listings/robustness.py.tex}
	}
	\vspace{-1.5em}
	\caption{Example of a local robustness property specified with \dsl.}
	\label{lst:localrobustness}
\end{wrapfigure}
\textbf{\dsl} Due to the lack of a standard format for specifying DNN properties,
we develop a Python-embedded DSL for DNN properties, which we call \dsl.
\dsl is designed to express any property that can be verified by existing DNN verifiers in a form that is independent of the network.
DNNP is described in more detail in Appendix~A of the extended version of this paper~\cite{shriver-etal:arxiv:2021:dnnv}.

We demonstrate \dsl with an example of a local robustness property, shown in \figref{lst:localrobustness}.
The property specifies that, for all inputs, \texttt{x_} (Lines 14-23), in the input space (Line 18) and within a hyper-rectangle of radius \texttt{e} centered at the given input \texttt{x} (Line 19), the network should predict the same maximum class for both \texttt{x_} and \texttt{x} (Line 21).
For Fashion MNIST, this means that for all images within an $L_\infty$ distance of $e$ (specified on Line 12) from image 1 of the dataset (selected on Lines 10-11), the network should classify all of these images the same as it does for image 1.
We first import several Python packages that will be useful for specifying the property (Lines 1-3), including the dataset used to train the network, and a method for data manipulation.
Because \dsl allows importing arbitrary Python packages, it enables re-use of the same data loading and manipulation methods used to train a network.
After importing the necessary utilities, we define several variables that will be used in the final property expression (Lines 5-12).
Two of these variables, $i$ on Line 10 and $e$ on Line 12 are declared as parameters, which allows them to be specified on the command line at run time.
The value for $e$ must be provided at run time, since no default value is provided.
Finally, we define the semantics of the property specification, using methods provided by \dsl, as well as variables defined above (Lines 14-23).

\textbf{Other Input Formats.}\label{subsec:inputformat:others} Since the creation of DNNV, two new input formats,
VNNLIB~\cite{VNNLIB} and SOCRATES~\cite{SOCRATES}, have emerged in an attempt to standardize the verifier input space.
The current draft of VNNLIB also uses ONNX as the DNN input format, however it supports a much smaller set of operations than DNNV, supporting only 17 ONNX operations.
The VNNLIB property format is a subset of SMTLIB in which variables of the form $X_i$ are implicitly mapped to network inputs and variables of the form $Y_i$ are implicitly mapped to network outputs.
In its current form, this specification only supports DNN models with a single flat input tensor and single flat output tensor, whereas DNNP and ONNX can support DNN models with multiple inputs and output tensors of any shape.
SOCRATES proposes JSON format containing both the property and network specifications.
Because DNNV treats networks and properties independently, properties can be re-used for multiple networks, and only a single network must be stored to check multiple properties, resulting in a lower storage cost, especially for large networks.
Additionally, while the custom JSON format used by SOCRATES requires new DNN translation tools to be written to convert models to the required format, the ONNX format used by DNNV is commonly available in most machine learning frameworks.
While we believe that ONNX and DNNP are currently the most expressive and easily accessible input formats currently proposed, DNNV can provide benefits to any format through DNN simplification and property reduction to increase the applicability of all verifiers.

\subsection{Network Simplification}

In order to allow verifiers to be applied to a wider range of real world networks, \tool provides tools for network simplification.
Network simplification takes in an operation graph and applies a set of semantics preserving transformations to the operation graph to remove unsupported structures, or  to transform sequences of operations into a single more commonly supported operation.

An operation graph $G_\DNN = \aset{V_\DNN, E_\DNN}$ is a directed graph where nodes, $v \in V_\DNN$ represent operations, and edges $e \in E_\DNN$ represent inputs to those operations.
Simplification, $\mathit{simplify}: \mathcal{G} \rightarrow \mathcal{G}$, transforms an operation graph $G_\DNN \in \mathcal{G}$, to an equivalent DNN with more commonly supported structure, $\mathit{simplify}(G_\DNN) = G_{\DNN'}$, such that the resulting DNN has the same behavior as the original $\forall x. \DNN(x) = \DNN'(x)$, and uses more commonly supported structures.

\begin{figure}[t]
	\centering
	\includegraphics[width=0.9\linewidth]{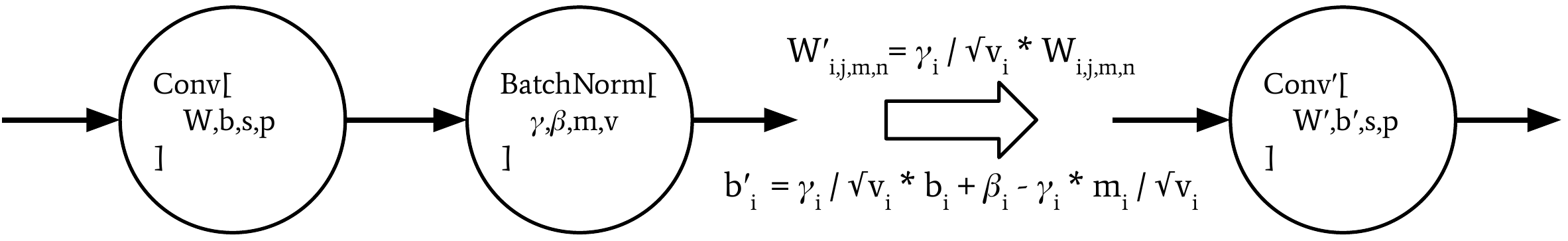}
	\vspace{-1em}
	\caption{
		Batch Normalization Simplification simplifies a batch norm following a convolution operation to an equivalent single convolution operation with modified weights and bias, while maintaining the strides and pads.
	}
	\label{fig:batchnormsimplification}
\end{figure}

One such simplification is \emph{batch normalization simplification}, which removes batch normalization operations from a network by combining them with a preceding convolution operation or generalized matrix multiplication (GEMM) operation.
This is possible since batch normalization, convolution, and GEMM operations are all affine operations.
The simplification of a batch normalization operation following a convolution operation is shown in \figref{fig:batchnormsimplification}.
If no applicable preceding layer exists, the batch normalization layer is converted into an equivalent convolution operation.
This simplification enables the application of verifiers without explicit support for batch normalization operations, such as \neurify and \marabou, to networks with these operations.

DNNV currently includes 6 additional
DNN simplifications, enumerated and described in more detail in Appendix~B of the extended version of this paper~\cite{shriver-etal:arxiv:2021:dnnv}.

\subsection{Property Reduction}

In order to allow verifiers to be applied to more general safety properties, \tool provides tools to reduce properties to a supported form.
For instance, properties can be translated to local robustness properties, which are required by \mipverify or reachability properties which are required by \reluplex.

Property reduction takes in a verification problem, which is comprised of a property specification and a network, and encodes it as an equivalid set of verification problems with properties in a form supported by a given verifier.

A \textit{verification problem} is a pair, $\psi = \aset{\DNN, \phi}$, of a DNN, $\DNN$, and a property specification $\phi$, formed to determine
whether $\DNN \models \phi$ is \emph{valid}.
Reduction, $reduce: \Psi \rightarrow P(\Psi)$, aims to transform a verification problem, $\aset{\DNN, \phi} = \psi \in \Psi$, to an equivalid
form, $reduce(\psi) = \set{\aset{\DNN_1, \phi_1}, \ldots, \aset{\DNN_k, \phi_k}}$, in which property specifications are in a common supported form.
As defined, reduction has two key properties.
The first property is that the set of resulting problems is equivalid with the original verification problem.
The second property is that the resulting set of problems all use the same property type.
Applying reduction enables verifiers to support a large set of verification problems by implementing support for a single property type.

\begin{figure}[t]
	\centering
	\includegraphics[width=0.9\linewidth]{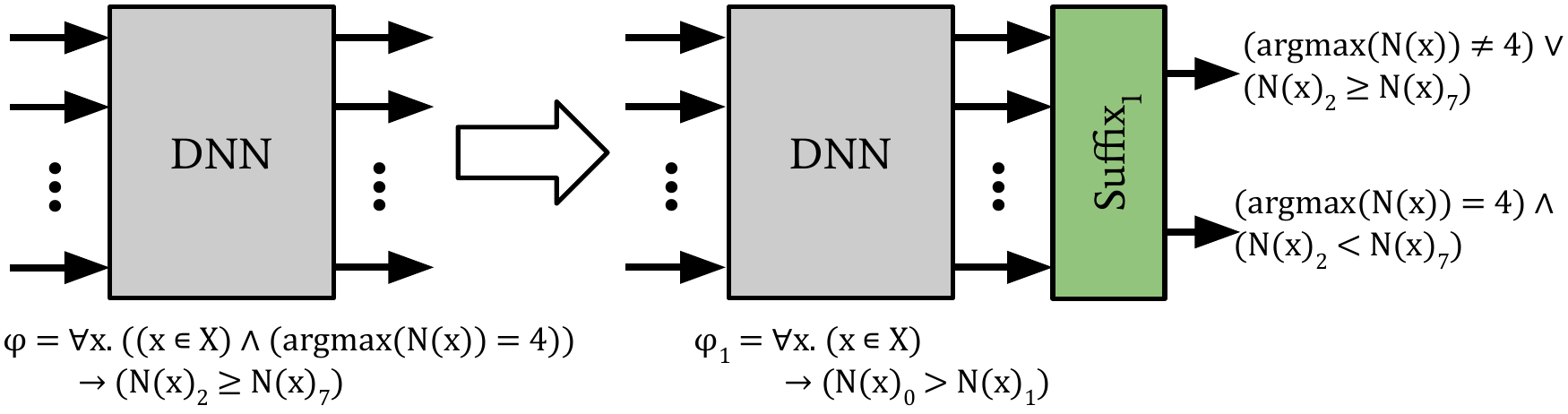}
	\vspace{-1em}
	\caption{Property reduction to a local robustness property adds a suffix that classifies outputs as violations or non-violations of the original output constraints, and changing the property to a common form of robustness property.}
	\label{fig:propertyreduction}
\end{figure}

For example, given a network that classifies images of clothing items, a user may want to specify that, if the network classifies an image as a coat, then the score given to the class of a pullover is not less than the score for the sneaker class.
The property is specified in the bottom left of \figref{fig:propertyreduction}.
Such a verification problem can be difficult to specify for many verifiers.
For example, \neurify would require writing code to specify linear constraints for the property and re-compiling the verifier, and \mipverify cannot support this property as is.
\tool can reduce this verification problem to an equivalent problem with a robustness property.

A high level overview of this reduction is shown in \figref{fig:propertyreduction};
a more detailed description is provided in Appendix~C of the extended version of this paper~\cite{shriver-etal:arxiv:2021:dnnv}.

\subsection{Input and Output Translation}

Because of the large variety of input formats required by the verifiers, one of the primary components of \tool translates from its internal representation of properties and networks to the input formats of each verifier.

\tool also requires an output translator that can parse the results of running a verifier and returns \texttt{sat}, \texttt{unsat}, or \texttt{unknown}.
If the result is \texttt{sat}, indicating a violation was found, \tool also returns a counter example to the property, and validates that it does violate the property by performing inference with the network and confirming that the input and output do not satisfy the property.

\section{Implementation}

\approach is written in 8400 lines of Python code and is available for download and re-use at \url{https://doi.org/10.5281/zenodo.4717922}.
Python was chosen due to its ubiquitous use for developing deep neural networks.
\tool currently supports 13 verifiers, and was designed to facilitate the integration of new verifiers. 
The currently supported verifiers are shown in \tabref{tab:verifiers}, along with their original input formats, and algorithmic approach.
Around 2000 LOC (of the 8400 total LOC) are used to integrate these 13 verifiers into \tool, with \planet requiring the most effort at 437 lines, and \bab and \babsb requiring the least effort with 89 lines of code due to re-use of the \planet input translator.

\subsection{Supporting Reuse and Extension}

\tool is designed to facilitate the integration of new verifiers.
The 5 primary components of \tool, DNN simplification, property reduction, input translation, verifier execution, and output translation are designed to be re-usable, and to facilitate the implementation of new components by providing utilities for traversing and manipulating operation graphs and properties.

Networks are represented as an operation graph, where nodes represent operations in the DNN and edges represent inputs and outputs to those operations. 
The operation graph can also be traversed using a visitor pattern. 
This pattern is particularly useful for the development of DNN simplifications and input translators.
It allows developers to easily traverse computation graphs in order to translate operations to the required format.
We provide built-in utilities for converting from our internal network representation to ONNX, PyTorch, and TensorFlow models.
The implementation also includes utilities for performing pattern matching on operation graphs.
We utilize this feature to provide utilities that transform a network from an operation graph representation to a sequential layer representation, which is particularly useful for the network input translator of \neurify, which requires DNNs to have a regular structure of a set of convolutional layers followed by fully connected layers, all with relu activations.

\subsection{Usage}

\tool can be run from the command line as follows:
\texttt{python -m dnnv <prop> <verifier> --network <name> <path>},
where the arguments correspond to a DNN model in the ONNX format, a property written
in \dsl, and the verifier to run.
Many additional options can be seen by specifying the \texttt{-h} option.

After execution,  
for each verifier, \tool reports the verification result as
one of \texttt{sat} (if the property was falsified), \texttt{unsat}
(if the property was proven to hold), \texttt{unknown} (if the
verifier is incomplete and could not prove the property holds),
or \texttt{error}, along with the reason for error, if an error
occurs during DNN and property translation, or during verifier
execution.
\tool also reports the time to translate and verify the property.

\section{Study}
\label{sec:study}

We now examine the applicability of verifiers to existing verification benchmarks with and without  \tool.
A verification benchmark consists of a set of verification problems which are used to evaluate the performance of a verifier.
A  problem is made of a DNN and a property specification and asks whether the property is valid for the given DNN.
We consider a verifier to support a benchmark if it can be run on that benchmark out of the box.
We consider a verifier to have support for a benchmark through \tool if \tool can be run on that benchmark with networks specified using \onnx and properties specified in \dsl, 
and can reduce, simplify, and translate the  problem to work with the target verifier.

\begin{table}[t]
	\centering
	\caption{Verifier benchmarks.}
	\label{tab:benchmarks}
	\begin{tabular}{lllcccccc}
		\toprule
		&                     &                                                                                                                                              & & & \multicolumn{4}{c}{\textbf{Features}} \\
		\textbf{Key} & \textbf{Name}       & \textbf{Uses}                                                                       & \textbf{\#P} & \textbf{\#$\DNN$} & \textbf{$\neg$HR} & \textbf{C} & \textbf{R} & \textbf{$\neg$ReLU} \\
		\midrule
		\textsc{AX}  & ACAS Xu             & \cite{katz-etal:CAV:2017:reluplex,bunel-etal:NIPS:2018:plnn,wang-etal:NIPS:2018:neurify,katz-etal:cav:2019:marabou,bak-etal:cav:2020:nnenum} & 10           & 45                        &               &               &                \\
		\textsc{CD}  & Collision Detection & \cite{ehlers:ATVA:2017:planet,bunel-etal:NIPS:2018:plnn,katz-etal:cav:2019:marabou} & 500          & 1                 &                   &            &            &                   \\
		\textsc{PM}  & \planet MNIST       & \cite{ehlers:ATVA:2017:planet}                                                      & 7            & 1                 & \checkmark        & \checkmark &            &                   \\
		\textsc{TS}  & TwinStream          & \cite{bunel-etal:arxiv:2017:plnnv1}                                                 & 1            & 81                &                   &            &            &                   \\
		\textsc{PCA} & PCAMNIST            & \cite{bunel-etal:NIPS:2018:plnn}                                                    & 12           & 17                &                   &            &            &                   \\
		\textsc{MM}  & \mipverify MNIST    & \cite{tjeng-etal:ICLR:2019:mipverify}                                               & 10000        & 5                 &                   & \checkmark &            &                   \\
		\textsc{MC}  & \mipverify CIFAR10  & \cite{tjeng-etal:ICLR:2019:mipverify}                                               & 10000        & 2                 &                   & \checkmark & \checkmark &                   \\
		\textsc{NM}  & \neurify MNIST      & \cite{wang-etal:NIPS:2018:neurify,henriksen-lomuscio:ECAI:2020:verinet}             & 500          & 4                 &                   & \checkmark &            &                   \\
		\textsc{NDb} & \neurify Drebin     & \cite{wang-etal:NIPS:2018:neurify}                                                  & 500          & 3                 &                   &            &            &                   \\
		\textsc{NDv} & \neurify DAVE       & \cite{wang-etal:NIPS:2018:neurify}                                                  & 200          & 1                 & \checkmark        & \checkmark &            &                   \\
		\textsc{DZM} & \deepzono MNIST     & \cite{singh-etal:NIPS:2018:deepzono}                                                & 1700         & 10                &                   & \checkmark & \checkmark & \checkmark        \\
		\textsc{DZC} & \deepzono CIFAR10   & \cite{singh-etal:NIPS:2018:deepzono}                                                & 1700         & 5                 &                   & \checkmark &            & \checkmark        \\
		\textsc{DPM} & \deeppoly MNIST     & \cite{singh-etal:POPL:2019:deeppoly,henriksen-lomuscio:ECAI:2020:verinet}           & 1500         & 8                 &                   & \checkmark &            & \checkmark        \\
		\textsc{DPC} & \deeppoly CIFAR10   & \cite{singh-etal:POPL:2019:deeppoly}                                                & 800          & 5                 &                   & \checkmark &            &                   \\
		\textsc{RZM} & \refinezono MNIST   & \cite{singh-etal:ICLR:2019:refinezono}                                              & 800          & 8                 &                   & \checkmark &            &                   \\
		\textsc{RZC} & \refinezono CIFAR10 & \cite{singh-etal:ICLR:2019:refinezono}                                              & 200          & 2                 &                   & \checkmark &            &                   \\
		\textsc{RPM} & \refinepoly MNIST   & \cite{singh-etal:NeurIPS:2019:refinepoly}                                           & 600          & 6                 &                   & \checkmark &            &                   \\
		\textsc{RPC} & \refinepoly CIFAR10 & \cite{singh-etal:NeurIPS:2019:refinepoly}                                           & 300          & 3                 &                   & \checkmark & \checkmark &                   \\
		\textsc{VC}  & \verinet CIFAR10    & \cite{henriksen-lomuscio:ECAI:2020:verinet}                                         & 250          & 1                 &                   & \checkmark &            &                   \\
		\bottomrule
	\end{tabular}
\end{table}

\noindent \textbf{Benchmarks.} To evaluate benchmark support, we collected the benchmarks used by each of the 13 verifiers supported by DNNV, 
and determined whether each verifier can run on the benchmark out of the box, and also whether they could be run on the benchmark when DNNV is applied. 
The verification benchmarks are shown in \tabref{tab:benchmarks} and are also described in more detail in Appendix~D of the extended version of this paper~\cite{shriver-etal:arxiv:2021:dnnv}.
Each row of the table corresponds to a benchmark, to which we assign a short key for identifying the benchmark.
For each benchmark, we give the name, some of the verifiers it  evaluated, the number of properties (\#P) and networks (\#$\DNN$), 
and features that can make it challenging for verifiers.
These features include
whether any properties cannot represent their input constraints using hyper-rectangles ($\neg$HR), 
whether any network in the benchmark contains convolution operations (C), 
whether any network contains residual structures (R), 
and whether any network uses any non-ReLU activation functions ($\neg$ReLU).

\begin{table}[t]
	\centering
	\caption{Benchmark support by each verifier. The left half of the circle is black if the verifier can support the benchmark out of the box, and is white otherwise. The right half is black if the verifier supports the benchmark through DNNV, 
	and is white otherwise. An absent circle indicates that the verifier can not be made to support some aspect of the benchmark.
	} 
	\label{tab:benchmarksupport}
	\begin{tabular}{lccccccccccccccccccc}
		\toprule
		& \multicolumn{19}{c}{\textbf{Benchmark}} \\
		\textbf{Verifier} & \rotatebox{90}{\textsc{AX}} & \rotatebox{90}{\textsc{CD}} & \rotatebox{90}{\textsc{PM}} & \rotatebox{90}{\textsc{TS}} & \rotatebox{90}{\textsc{PCA}} & \rotatebox{90}{\textsc{MM}} & \rotatebox{90}{\textsc{MC}} & \rotatebox{90}{\textsc{NM}} & \rotatebox{90}{\textsc{NDb}} & \rotatebox{90}{\textsc{NDv}} & \rotatebox{90}{\textsc{DZM}} & \rotatebox{90}{\textsc{DZC}} & \rotatebox{90}{\textsc{DPM}} & \rotatebox{90}{\textsc{DPC}} & \rotatebox{90}{\textsc{RZM}} & \rotatebox{90}{\textsc{RZC}} & \rotatebox{90}{\textsc{RPM}} & \rotatebox{90}{\textsc{RPC}} & \rotatebox{90}{\textsc{VC}} \\
		\midrule
		\reluplex         & \filledcirc{full}           & \filledcirc{left}           & \filledcirc{empty}          & \filledcirc{full}           & \filledcirc{full}            & \filledcirc{empty}          & \filledcirc{empty}          & \filledcirc{empty}          & \filledcirc{right}           & \filledcirc{empty}           &                              &                              &                              & \filledcirc{empty}           & \filledcirc{empty}           & \filledcirc{empty}           & \filledcirc{empty}           & \filledcirc{empty}           & \filledcirc{empty}          \\
		\planet           & \filledcirc{full}           & \filledcirc{left}           & \filledcirc{full}       & \filledcirc{full}           & \filledcirc{full}            & \filledcirc{right}          & \filledcirc{right}          & \filledcirc{right}          & \filledcirc{right}           & \filledcirc{right}       &                              &                              &                              & \filledcirc{right}           & \filledcirc{right}           & \filledcirc{right}           & \filledcirc{right}           & \filledcirc{right}           & \filledcirc{right}          \\
		\bab              & \filledcirc{full}           & \filledcirc{left}           & \filledcirc{full}       & \filledcirc{full}           & \filledcirc{full}            & \filledcirc{right}          & \filledcirc{empty}          & \filledcirc{right}          & \filledcirc{right}           & \filledcirc{right}       &                              &                              &                              & \filledcirc{right}           & \filledcirc{right}           & \filledcirc{right}           & \filledcirc{right}           & \filledcirc{empty}           & \filledcirc{right}          \\
		\babsb            & \filledcirc{full}           & \filledcirc{left}           & \filledcirc{full}       & \filledcirc{full}           & \filledcirc{full}            & \filledcirc{right}          & \filledcirc{empty}          & \filledcirc{right}          & \filledcirc{right}           & \filledcirc{right}       &                              &                              &                              & \filledcirc{right}           & \filledcirc{right}           & \filledcirc{right}           & \filledcirc{right}           & \filledcirc{empty}           & \filledcirc{right}          \\
		\mipverify        & \filledcirc{right}          & \filledcirc{empty}          & \filledcirc{empty}          & \filledcirc{right}          & \filledcirc{right}           & \filledcirc{full}           & \filledcirc{left}           & \filledcirc{empty}          & \filledcirc{right}           & \filledcirc{empty}           &                              &                              &                              & \filledcirc{empty}           & \filledcirc{empty}           & \filledcirc{empty}           & \filledcirc{empty}           & \filledcirc{empty}           & \filledcirc{empty}          \\
		\neurify          & \filledcirc{full}           & \filledcirc{empty}          & \filledcirc{right}      & \filledcirc{right}          & \filledcirc{right}           & \filledcirc{right}          & \filledcirc{empty}          & \filledcirc{full}           & \filledcirc{full}            & \filledcirc{full}        &                              &                              &                              & \filledcirc{right}           & \filledcirc{right}           & \filledcirc{right}           & \filledcirc{right}           & \filledcirc{empty}           & \filledcirc{right}          \\
		\deepzono         & \filledcirc{full}           & \filledcirc{empty}          & \filledcirc{empty}      & \filledcirc{right}          & \filledcirc{right}           & \filledcirc{right}          & \filledcirc{right}          & \filledcirc{right}          & \filledcirc{right}           & \filledcirc{empty}       & \filledcirc{full}            & \filledcirc{full}            & \filledcirc{full}            & \filledcirc{full}            & \filledcirc{full}            & \filledcirc{full}            & \filledcirc{full}            & \filledcirc{full}            & \filledcirc{right}          \\
		\deeppoly         & \filledcirc{full}           & \filledcirc{empty}          & \filledcirc{empty}      & \filledcirc{right}          & \filledcirc{right}           & \filledcirc{right}          & \filledcirc{right}          & \filledcirc{right}          & \filledcirc{right}           & \filledcirc{empty}       & \filledcirc{full}            & \filledcirc{full}            & \filledcirc{full}            & \filledcirc{full}            & \filledcirc{full}            & \filledcirc{full}            & \filledcirc{full}            & \filledcirc{full}            & \filledcirc{right}          \\
		\refinezono       & \filledcirc{full}           & \filledcirc{empty}          & \filledcirc{empty}      & \filledcirc{right}          & \filledcirc{right}           & \filledcirc{right}          & \filledcirc{right}          & \filledcirc{right}          & \filledcirc{right}           & \filledcirc{empty}       & \filledcirc{full}            & \filledcirc{full}            & \filledcirc{full}            & \filledcirc{full}            & \filledcirc{full}            & \filledcirc{full}            & \filledcirc{full}            & \filledcirc{full}            & \filledcirc{right}          \\
		\refinepoly       & \filledcirc{full}           & \filledcirc{empty}          & \filledcirc{empty}      & \filledcirc{right}          & \filledcirc{right}           & \filledcirc{right}          & \filledcirc{right}          & \filledcirc{right}          & \filledcirc{right}           & \filledcirc{empty}       & \filledcirc{full}            & \filledcirc{full}            & \filledcirc{full}            & \filledcirc{full}            & \filledcirc{full}            & \filledcirc{full}            & \filledcirc{full}            & \filledcirc{full}            & \filledcirc{right}          \\
		\marabou          & \filledcirc{full}           & \filledcirc{left}           & \filledcirc{right}      & \filledcirc{full}           & \filledcirc{full}            & \filledcirc{right}          & \filledcirc{empty}          & \filledcirc{right}          & \filledcirc{right}           & \filledcirc{right}       &                              &                              &                              & \filledcirc{right}           & \filledcirc{right}           & \filledcirc{right}           & \filledcirc{right}           & \filledcirc{empty}           & \filledcirc{right}          \\
		\nnenum           & \filledcirc{full}           & \filledcirc{empty}          & \filledcirc{right}      & \filledcirc{right}          & \filledcirc{right}           & \filledcirc{right}          & \filledcirc{empty}          & \filledcirc{right}          & \filledcirc{right}           & \filledcirc{right}       &                              &                              &                              & \filledcirc{right}           & \filledcirc{right}           & \filledcirc{right}           & \filledcirc{right}           & \filledcirc{empty}           & \filledcirc{right}          \\
		\verinet          & \filledcirc{right}          & \filledcirc{empty}          & \filledcirc{empty}      & \filledcirc{right}          & \filledcirc{right}           & \filledcirc{right}          & \filledcirc{empty}          & \filledcirc{full}           & \filledcirc{right}           & \filledcirc{empty}       & \filledcirc{empty}           & \filledcirc{right}           & \filledcirc{right}           & \filledcirc{right}           & \filledcirc{right}           & \filledcirc{right}           & \filledcirc{right}           & \filledcirc{empty}           & \filledcirc{full}           \\
		\bottomrule
	\end{tabular}
\end{table}

\noindent \textbf{Results.} The support of verifiers for each benchmark is shown in \tabref{tab:benchmarksupport}.
Each row of this table corresponds to one of the 13 verifiers supported by \tool, and each column corresponds to one of the 19 benchmarks identified in \tabref{tab:benchmarks}.
Each cell of the table may contain a circle that identifies the support of the verifier for the benchmark.
The left half of the circle is black if the verifier can support the benchmark out of the box, and is white otherwise. 
The right half is black if the verifier supports the benchmark through DNNV, 
and white otherwise.
An absent circle indicates that the verifier can not be made to support some aspect of the benchmark.
For the benchmarks shown here, this is always due to the presence of non-ReLU activation functions in some of the networks in the benchmarks.

As shown in \tabref{tab:benchmarksupport}, \tool can dramatically increase the support of verifiers for benchmarks.
For example, the \planet verifier could originally be run on 5 of the 19 benchmarks, but could be run on 16 using \tool.
Similarly, the \nnenum verifier, could originally only be run on 1 of the existing benchmarks, but could be run on 13 using \tool.
\textbf{Of the 223 pairs of verifiers and benchmarks for which support may be possible, 166 of them are currently supported by \tool, an increase of over 2.4 times the 68 pairs supported without \tool.}

\section{Conclusion}

We present  the \tool framework for reducing the burden on DNN verifier researchers, developers, and users. \tool standardizes input and output formats, includes a simple yet expressive DSL for specifying DNN properties, and provides powerful simplification and reduction operations to  facilitate the application, development, and comparison of DNN verifiers. Our study showed the potential of \tool and we made its implementation available, with support for 13 verifiers, and extensive documentation.

\section{Acknowledgment}
This material is based in part upon work supported by the
National Science Foundation under Grant Number 1900676 and 2019239.

\newpage

\appendix

\section{DNNP}
\label{appendix:dnnp}

A property specification defines the desired behavior of a DNN in a formal language. DNNV uses a custom Python-embedded DSL for writing property specifications, which we call DNNP.
Embedding \dsl in Python allows for the rich ecosystem of the host language to be used in writing specifications~\cite{hudak:1998:dsl}.
However, DNNV is still of a work-in-progress, so some expressions (such as star expressions) are not yet supported by our property parser.
We are still working to fully support all Python expressions, but the current version supports the most common use cases.

\begin{figure}
    \centering
    \scriptsize
    \begin{grammar}
        <property> ::= <python-imports> <assignment-list> <expr>

        <python-imports> ::= `' \alt <python-imports> `import' <id>
        \alt <python-imports> `import' <id> `as' <id>
        \alt <python-imports> `from' <id> `import' <id>

        <assignment-list> ::= `' \alt <assignment-list> <assignment>

        <assignment> ::= <id> `=' <expr>

        <expr> ::= `Forall(' <id> `,' <expr> `)'
        \alt `And(' <expr-seq> `)'
        \alt `Or(' <expr-seq> `)'
        \alt `Implies(' <expr> `,' <expr> `)'
        \alt `Parameter(' <id> `, type=' <id> `)'
        \alt ...
        \alt <python-expr>

        <expr-seq> ::= <expr> | <expr-seq> `,' <expr>
    \end{grammar}
    \vspace{-1.5em}
    \caption{Subset of the grammar for \dsl.
    }
    \label{fig:dsl}
\end{figure}

\figref{fig:dsl} shows the definition of the \dsl grammar.
The general structure of a property specification is as follows:

\begin{enumerate}
    \item A set of python module imports
    \item A set of variable definitions
    \item A property expression
\end{enumerate}

\subsection{Imports}

Imports have the same syntax as Python import statements, and they can be used to import arbitrary Python modules and packages.
This allows re-use of datasets or input pre-processing code.
For example, the Python package \texttt{numpy} can be imported to load a dataset.
Inputs can then be selected from the dataset, or statistics, such as the mean data point, can be computed on the fly.

\subsection{Definitions}

After any imports, DNNP allows a sequence of assignments to define variables that can be used in the final property specification.
For example, \texttt{i = 0}, will define the variable \texttt{i} to a value of 0.

These definitions can be used to load data and configuration parameters, or to alias expressions that may be used in the property formula.
For example, if the \texttt{torchvision.datasets} package has been imported, then \texttt{data = datasets.MNIST("/tmp")} will define a variable \texttt{data} referencing the MNIST dataset from this package.
Additionally, the \texttt{Parameter} class can be used to declare parameters that can be specified at run time.
\texttt{eps = Parameter("epsilon", type=float)}, will define the variable \texttt{eps} to have type float and will expect a value to be specified at run time. This value can be specified to DNNV with the option \texttt{--prop.epsilon}.

Definitions can also assign expressions to variables to be used in the property specification later.
For example, \texttt{x_in_unit_hyper_cube = 0 <= x <= 1} can be used to assign an expression specifying that the variable x is within the unit hyper cube to a variable.
This could be useful for more complex properties with a lot of redundant sub-expressions.

A network can be defined using the \texttt{Network} class.
\texttt{N = Network("N")}, specifies a network with the name \texttt{N} (which is used at run time to concretize the network with a specific DNN model).
All networks with the same name refer to the same model.

\subsection{Property Expression}

Finally, the last part of the property specification is the property formula itself.
It must appear at the end of the property specification.
All statements before the property formula must be either import or assignment statements.

The property formula defines the desired behavior of the DNN in a subset of first-order-logic.
It can make use of arbitrary Python code, as well as any of the expressions defined before it.

DNNP provides many functions to define expressions.
The function \texttt{Forall( symbol, expression)} can be used to specify that the provided expression is valid for all values of the specified symbol.
The function \texttt{And(*expression)}, specifies that all of the expressions passed as arguments to the function must be valid.
\texttt{And(expr1, expr2)} can be equivalently specified as \texttt{expr1 \& expr2}.
The function \texttt{Or(*expression)}, specifies that at least one of the expressions passed as arguments to the function must be valid.
\texttt{Or(expr1, expr2)} can be equivalently specified as \texttt{expr1 | expr2}.
The function \texttt{Implies(expression1, expression2)}, specifies that if \texttt{expression1} is true, then \texttt{expression2} must also be true.
The \texttt{argmin} and \texttt{argmax} functions can be used to get the argmin or argmax value of a network’s output, respectively.

In property expressions, networks can be called like functions to get the outputs for the network for a given input. Networks can be applied to symbolic variables (such as universally quantified variables), as well as numpy arrays.

\section{DNN Simplifications}
\label{appendix:dnnsimplifications}

In this section, we describe the DNN simplifications currently performed by DNNV.
This is not a full list of all possible simplifications, but have been useful for some networks we have encountered in practice.

\subsection{BatchNormalization Simplification}

BatchNormalization simplification removes BatchNormalization operations from a network by combining them with a preceeding Conv operation or Gemm operation. If no applicable preceeding layer exists, the batch normalization layer is converted into an equivalent Conv operation. This simplification can decrease the number of operations in the model and increase verifier support, since many verifiers do not support BatchNormalization operations.

\subsection{Identity Removal}

DNNV removes many types of identity operations from DNN models, including explicit Identity operations, Concat operations with a single input, and Flatten operations applied to flat tensors. Such operations can occur in DNN models due to user error, or through automated processes, and their removal does not affect model behavior.

\subsection{Convert MatMul followed by Add to Gemm}

DNNV converts instances of MatMul (matrix multiplication) operations, followed immediately by Add operations to an equivalent Gemm (generalized matrix multiplication) operation. The Gemm operation generalizes the matrix multiplication and addition, and can simplify subsequent processing and analysis of the DNN.

\subsection{Combine Consecutive Gemm}

DNNV combines two consecutive Gemm operations into a single equivalent Gemm operation, reducing the number of operations in the DNN.

\subsection{Combine Consecutive Conv}

In special cases, DNNV can combine consecutive Conv (convolution) operations into a single equivalent Conv operation, reducing the number of operations in the DNN.
Currently, DNNV can combine Conv operations when the first Conv uses a diagonal 1 by 1 kernel with a stride of 1 and no zero padding, and the second Conv has no zero padding. This case can occur after converting a normalization layer (such as BatchNormalization) to a Conv operation.

\subsection{Bundle Pad}

DNNV can bundle explicit Pad operations with an immediately succeeding Conv or MaxPool operation. This both simplifies the DNN model, and increases support, since many verifiers do not support explicit Pad operations (but can support padding as part of a Conv or MaxPool operation).

\subsection{Move Activations Backward}

DNNV moves activation functions through reshaping operations to immediately succeed the most recent non-reshaping operation. This is possible since activation functions are element-wise operations. This transformation can simplify pattern matching in later analysis steps by reducing the number of possible patterns.

\section{Property Reduction}
\label{appendix:propreduction}

In this section, we provide the algorithm for reducing properties to reachability properties, as well as proofs for the equivalidity of the resulting set of reachability properties and original property.
Algorithm~\ref{alg:reduction} is the overall reduction algorithm, while Algorithm~\ref{alg:disjuncttohpoly} and \ref{alg:constructsuffix} are subprocedures used by the main algorithm.
The algorithm and proofs for reduction to other property types (such as robustness) are very similar.

We assume that properties are of the form $\forall{x\in \mathbb{R}^{n}}: \phi_{\mathcal{X}}(x) \rightarrow \phi_{\mathcal{Y}}(\DNN(x))$, where $\phi_{\mathcal{X}}$ is a set of constraints over the inputs -- the pre-condition, and $\phi_{\mathcal{Y}}$ is a set of constraints over the outputs -- the post-condition.
We also assume that constraints are represented as linear inequalities.

\begin{algorithm}[ht]
	\small
	\DontPrintSemicolon
	\KwIn{Correctness problem $\langle \DNN, \phi \rangle$}
	\KwOut{A set of robustenss problems $\set{\langle \DNN_1, \phi_1 \rangle, ..., \langle \DNN_i, \phi_i \rangle}$}
	\Begin{
		$\phi' \leftarrow DNF(\neg\phi)$\;\label{algln:negdnf}
		$\Psi \leftarrow \emptyset$\;
		\For{$\mathit{disjunct} \in \phi'$}{
			$\phi_{\mathcal{X}} \leftarrow \mathrm{extract\_input\_constraints}(disjunct)$\;
			$\phi_{\mathcal{Y}} \leftarrow \mathrm{extract\_output\_constraints}(disjunct)$\;
			$hspoly \leftarrow \mathrm{disjunct\_to\_hpolytope}(\phi_{\mathcal{Y}})$\;\label{algln:hpoly}
			$suffix \leftarrow \mathrm{construct\_suffix}(hspoly)$\;\label{algln:constructsuffix}
			$\DNN' \leftarrow \mathit{suffix} \circ \DNN$\;\label{algln:constructdnn}
			$\phi' \leftarrow \forall x.(x \in \phi_{\mathcal{X}} \implies \DNN'(x)_0 > \DNN'(x)_1)$\;\label{algln:constructphi}
			$\Psi \leftarrow \Psi \cup \langle \DNN', \phi' \rangle$\label{algln:constructpsi}
		}
		\Return{$\Psi$}\;
	}
	\caption{Property Reduction\label{alg:reduction}}
\end{algorithm}

\begin{algorithm}[ht]
	\footnotesize
	\DontPrintSemicolon
	\KwIn{Conjunction of linear inequalities $\phi_i$}
	\KwOut{Halfspace polytope $H$}
	\Begin{
		$H \leftarrow (A, b)$ where $A$ is an $(|\phi_i|)\times(m)$ matrix where columns correspond to the output variables $N(x)_0$ to $N(x)_{m-1}$\;
		\For{$ineq_j \in \phi_i$}{
			\If{$ineq_j$ uses $\geq$}{\label{algln:geq2leqstart}
				swap lhs and rhs;	switch inequality to $\leq$\;
			}
			\ElseIf{$ineq_j$ uses $>$}{
				swap lhs and rhs;	switch inequality to $<$\;
			}\label{algln:gt2ltend}
			move variables to lhs\label{algln:var2lhs}; move constants to rhs\;\label{algln:const2rhs}
			\If{$ineq_j$ uses $<$}{\label{algln:lt2leqstart}
				decrement rhs; switch inequality to $\leq$\;
			}\label{algln:lt2leqend}
			$A_{j} \leftarrow$ coefficients of variables on lhs\;\label{algln:lhs2A}
			$b_{j} \leftarrow$ rhs constant\;\label{algln:rhs2b}
		}
		\Return{$H$}\;
	}
	\caption{disjunct\_to\_hpolytope\label{alg:disjuncttohpoly}}
\end{algorithm}

\begin{algorithm}[ht]
	\footnotesize
	\DontPrintSemicolon
	\KwIn{Halfspace polytope $H = (A, b)$}
	\KwOut{A DNN with 2 fully connected layers $S$}
	\Begin{
		$S_h \leftarrow \mathrm{ReLU}(\mathrm{FullyConnectedLayer}(A, -b))$\;\label{algln:hiddensuffix}
		$W \leftarrow
			\begin{bmatrix}
				1 & 1 & ... & 1 \\
				0 & 0 & ... & 0
			\end{bmatrix}$\;
		$S_o \leftarrow \mathrm{FullyConnectedLayer}(W, \vec{0})$\;
		$S \leftarrow S_o \circ S_h$\;
		\Return{$S$}\;
	}
	\caption{construct\_suffix\label{alg:constructsuffix}}
\end{algorithm}

\subsection{Proofs}

In order to prove that the property reduction produces a set of correctness problems equivalid to the original problem, we first prove the following lemmas:

\begin{lemma}\label{lemma:halfspaceconstruction}
	Let $\phi$ be a conjunction of linear inequalities over the variables $x_i$ for i from $0$ to $n-1$.
	We can construct a halfspace polytope $H = (A, b)$ with Algorithm~\ref{alg:disjuncttohpoly} such that $(Ax \leq b) \Leftrightarrow (x \models \phi)$.
\end{lemma}

\begin{proof}
	We first show that every linear inequality in the conjunction can be reformulated to the form $a_0x_0 + a_1x_1 + ... + a_{n-1}x_{n-1} \leq b$.
	It is trivial to show that inequalities with a $\geq$ comparison can be manipulated to an equivalent form with $\leq$, and $>$ can be manipulated to become $<$.
	It is also trivial to show that the inequality can be manipulated to have variables on lhs and a constant value on rhs.
	This results in a conjunction of linear inequalities of the form $a_0x_0 + a_1x_1 + ... + a_{n-1}x_{n-1} < b$ and $a_0x_0 + a_1x_1 + ... + a_{n-1}x_{n-1} \leq b$.
	Finally, the $<$ comparison can be changed to a $\leq$ comparison by decrementing the constant on the right-hand-side from $b$ to $b'$ where $b'$ is the largest representable number less than $b$.

	We prove that linear inequalities using the $<$ comparison can be reformulated to use a $\leq$ comparison using a proof by contradiction.
	Assume that either $a_0x_0 + a_1x_1 + ... + a_{n-1}x_{n-1} < b$ and $a_0x_0 + a_1x_1 + ... + a_{n-1}x_{n-1} > b'$ or $a_0x_0 + a_1x_1 + ... + a_{n-1}x_{n-1} \geq b$ and $a_0x_0 + a_1x_1 + ... + a_{n-1}x_{n-1} \leq b'$.
	Then one of two cases must be true.
	Either $b' < a_0x_0 + a_1x_1 + ... + a_{n-1}x_{n-1} < b$, a contradiction, since $a_0x_0 + a_1x_1 + ... + a_{n-1}x_{n-1}$ cannot be both larger than the largest representable number less than $b$ and also less than $b$.\footnote{We further discuss the assumption that such a number exists in Section~\ref{sec:largestrepresentablenumber}.}
	Or $b \leq a_0x_0 + a_1x_1 + ... + a_{n-1}x_{n-1} \leq b'$, a contradiction, since $b' < b$ by definition.

	Given a conjunction of linear inequalities in the form $a_0x_0 + a_1x_1 + ... + a_{n-1}x_{n-1} \leq b$, Algorithm~\ref{alg:disjuncttohpoly} constructs $A$ and $b$ with a row in $A$ and value in $b$ corresponding to each conjunct.
	There are two cases to prove: $(Ax \leq b) \rightarrow (x \models \phi)$ and $(x \models \phi) \rightarrow (Ax \leq b)$.

	We prove case 1 by contradiction.
	Assume $(Ax \leq b)$ and $(x \not\models \phi)$.
	By the construction of $H$ in Algorithm~\ref{alg:disjuncttohpoly}, each conjunct of $\phi$ is exactly 1 constraint in $H$.
	If $Ax \leq b$, then all constraints in $H$ must be satisifed, and thus all conjuncts in $\phi$ must be satisfied and $x \models \phi$, a contradiction.

	We prove case 2 by contradiction.
	Assume $(x \models \phi)$ and $(Ax \not\leq b)$.
	By the construction of $H$ in Algorithm~\ref{alg:disjuncttohpoly}, each conjunct of $\phi$ is exactly 1 constraint in $H$.
	If $x \models \phi$, then all conjuncts in $\phi$ must be satisfied, and thus all constraints in $H$ must be satisifed and $Ax \leq b$, a contradiction.
\end{proof}

\begin{lemma}\label{lemma:suffixconstruction}
	Let $H = (A, b)$ be a halfspace polytope such that $Ax \leq b$. Then, a DNN, $\DNN_s$, can be built with Algorithm~\ref{alg:constructsuffix} that classifies whether its outputs satisfy $A(\DNN(x)) \leq b$ or not. Formally, $\DNN(x) \in H \Leftrightarrow \DNN_s(x)_0 \leq \DNN_s(x)_1$.
\end{lemma}

\begin{proof}
	There are 2 cases:
	\begin{enumerate}
		\item $\DNN(x) \in H \rightarrow \DNN_s(x)_0 \leq \DNN_s(x)_1$
		\item $\DNN_s(x)_0 \leq \DNN_s(x)_1 \rightarrow \DNN(x) \in H$
	\end{enumerate}

	We prove case 1 by contradiction.
	Assume $\DNN(x) \in H$ and $\DNN_s(x)_0 > \DNN_s(x)_1$.
	From Algorithm~\ref{alg:constructsuffix}, each neuron in the hidden layer of $\DNN_s$ corresponds to one constraint in $H$.
	The weights of each neuron are the values in the corresponding row of $A$, and the bias is the negation of the corresponding value of $b$.
	If the output $\DNN(x)$ satisfies the constraint, then the value of the neuron will be less than or equal to 0, otherwise it will be greater than 0.
	After application of the ReLU activation function, all neurons will be equal to 0 if their corresponding constraint is satisfied by $\DNN(x)$ and greater than 0 otherwise.
	The first neuron in the final layer sums all of the neurons in the hidden layer, while the second neuron has a constant value of 0.
	If $\DNN(x) \in H$, then all neurons in the hidden layer after activation must have a value of 0 since all constraints are satisfied.
	However, if all neurons have a value of 0, then their sum must also have a value of zero, and therefore $\DNN_s(x)_0 = \DNN_s(x)_1$, a contradiction.

	We prove case 2 by contradiction.
	Assume $\DNN_s(x)_0 \leq \DNN_s(x)_1$ and $\DNN(x) \not\in H$.
	From Algorithm~\ref{alg:constructsuffix}, each neuron in the hidden layer of $\DNN_s$ corresponds to one constraint in $H$.
	The weights of each neuron are the values in the corresponding row of $A$, and the bias is the negation of the corresponding value of $b$.
	If the output $\DNN(x)$ satisfies the constraint, then the value of the neuron will be less than or equal to 0, otherwise it will be greater than 0.
	After application of the ReLU activation function, all neurons will be equal to 0 if their corresponding constraint is satisfied by $\DNN(x)$ and greater than 0 otherwise.
	The first neuron in the final layer sums all of the neurons in the hidden layer, while the second neuron has a constant value of 0.
	If $\DNN(x) \not\in H$, then at least one neurons in the hidden layer after activation must have a value greater than 0 since at least one constraint is not satisfied.
	However, if any neuron has a value greater than 0, then their sum must also have a value greater than zero, and therefore $\DNN_s(x)_0 > \DNN_s(x)_1$, a contradiction.
\end{proof}

\begin{theorem}\label{thm:equivalid}
	Let $\psi = \aset{\DNN, \phi}$ be an arbitrary correctness problem with a DNN correctness property defined as a formula of disjunctions and conjunctions of linear inequalities over the input and output variables of $\DNN$.
	Property Reduction (Algorithm~\ref{alg:reduction}) maps $\psi$ to an equivalid set of correctness problems $reduce(\psi) = \set{\aset{\DNN_1, \phi_1}, \ldots, \aset{\DNN_k, \phi_k}}$.
	$$\DNN \models \psi \Leftrightarrow \forall \aset{\DNN_i, \phi_i} \in reduce(\psi) . \DNN_i \models \phi_i$$
\end{theorem}

\begin{proof}
	A model that satisfies any disjunct of $DNF(\neg\phi)$ falsifies $\phi$.
	If $\phi$ is falsifiable, then at least one disjunct of $DNF(\neg\phi)$ is satisfiable.

	Algorithm~\ref{alg:reduction} constructs a correctness problem for each disjunct.
	For each disjunct, Algorithm~\ref{alg:reduction} constructs a halfspace polytope, $H$, which is used to construct a suffix network, $\DNN_s$.
	The algorithm then constructs the network $\DNN'(x) = \DNN_s(\DNN(x))$.
	Algorithm~\ref{alg:reduction} pairs each constructed network with the property $\phi = \forall x . x \in [0, 1]^n \rightarrow \DNN'(x)_0 > \DNN'(x)_1$.
	A violation occurs only when $\DNN'(x)_0 \leq \DNN'(x)_1$.
	By Lemmas~\ref{lemma:halfspaceconstruction} and \ref{lemma:suffixconstruction}, we get that $\DNN'(x)_0 \leq \DNN'(x)_1$ if and only if $\DNN'(x) \in H$.
	If $\DNN'(x) \in H$ then $\DNN'(x)$ satisfies the disjunct and is therefore a violation of the original property.
\end{proof}

\subsection{On the Existance of a Bounded Largest Representable Number}
\label{sec:largestrepresentablenumber}

Our proof that property reduction generates a set of robustness problems equivalid to an arbitrary problem relies on the assumption that strict inequalities can be converted to non-strict inequalities.
To do so we rely on the existance of a largest representable number that is less than some given value.
While this is not necessarily true for all sets of numbers (\eg $\mathbb{R}$), it is true for for most numeric representations used in computation (\eg IEEE 754 floating point arithmetic).

\section{Verification Benchmarks}
\label{appendix:benchmarks}

We examine the benchmarks used to evaluate each of the 13 verifiers supported by DNNV,
and determine whether each verifier can run on the benchmark out of the box, and also whether they could be run on the benchmark when DNNV is applied.
Here we provide a short description of each of the 19 verification benchmarks that we have identified.
A short summary of some of the features of each verifier relevant to DNNV are shown in Table~\ref{tab:benchmarks}.
These features include
whether any properties cannot represent their input constraints using hyper-rectangles ($\neg$HR),
whether any network in the benchmark contains convolution operations (C),
whether any network contains residual structures (R),
and whether any network uses any non-ReLU activation functions ($\neg$ReLU).

The ACAS Xu (\textsc{AX}) benchmark, introduced for \reluplex~\cite{katz-etal:CAV:2017:reluplex}, is one of the most used verification benchmarks~\cite{bunel-etal:NIPS:2018:plnn,wang-etal:NIPS:2018:neurify,katz-etal:cav:2019:marabou,bak-etal:cav:2020:nnenum}.
The benchmark consists of 10 properties.
Property $\phi_1$ is a reachability property, specifying an upper bound on one of the 5 output variables.
Properties $\phi_5$, $\phi_6$, $\phi_9$, and $\phi_{10}$ are all traditional class robustness properties, specifying the desired class for the given input region.
Properties $\phi_3$, $\phi_4$, $\phi_7$ and $\phi_8$ are reachability properties, specifying a set of acceptable classes for the input region.
Properties $\phi_2$ is also a reachability property, specifying that a given output value cannot be greater than all others.
Each of the properties are applied to a subset of 45 networks trained on an aircraft collision avoidance dataset, with 5 inputs, 5 output classes and 6 layers of 50 neurons each.
The original benchmark included networks in \reluplex-NNET format, and a custom version of \reluplex was written for each property.
Later uses of the benchmark translated the verification problems into RLV format, which is used by \planet, \bab, and \babsb, as well as translating the networks into \onnx.
The benchmark in \onnx and \dsl format is fully supported by \tool.

The Collision Detection (\textsc{CD}) benchmark~\cite{ehlers:ATVA:2017:planet}, intoduced for the evaluation of \planet, consists of 500 local robustness properties for an 80 neuron network with a fully connected layer and max pooling layer that classifies whether 2 simulated vehicles will collide, given their current state.
The verification problems, in RLV format, are supported by \planet, \bab, and \babsb. The problems have also been modified to convert max pooling operations to a sequence of fully-connected layers with ReLU activations, and then translated to \reluplex-NNET format, enabling off the shelf support by \marabou, and a generalized version of \reluplex.
This benchmark is one of the few that is not supported by \tool, since the network contains structures that are not easily supported by ONNX.
In particular, the max-pooling operation in the original network, applied to a flat tensor, cannot be encoded by \onnx from their original format.

The \planet MNIST (\textsc{PM}) benchmark~\cite{ehlers:ATVA:2017:planet} is a set of 7 properties over a convolutional network trained on the MNIST dataset~\cite{lecun-etal::2010:MNIST}.
The first 4 of these are reachability properties with hyper-rectangle input constraints, while the next 2 are local robustness properties with hyper-rectangle input constraints, and the final property is an local robustness property with halfspace-polytope input constraints.
The original benchmark was provided in RLV format.
The first 6 of these properties are currently supported by \tool, while the final property could be supported by \tool with additional engineering effort.

The TwinStream (\textsc{TS}) benchmark~\cite{bunel-etal:arxiv:2017:plnnv1} consists of 1 property applied to 81 networks that output a constant value.
The property asserts that for all inputs, the output of the network is positive.
The original benchmark was provided in RLV format.
This benchmark is fully supported by \tool for all verifiers.

The PCAMNIST (\textsc{PCA}) benchmark~\cite{bunel-etal:NIPS:2018:plnn} consists of 12 reachability properties applied to 17 networks trained on modified versions of the MNIST dataset to predict the parity of the digit represented by the first $k$ components of the PCA decomposition of an image.
The original benchmark was provided in RLV format.
This benchmark is fully supported by \tool for all verifiers.

\mipverify MNIST (\textsc{MM}) consists of 10000 local robustness properties applied to 5 networks trained on the MNIST dataset.
The networks have varied structures: 2 networks are fully connected and 3 are convolutional.
We could not find an available version of the benchmark used by \mipverify to evaluate its original input format.
This benchmark is fully supported by \tool for all verifiers except \reluplex, which does not support convolution operations.

\mipverify CIFAR (\textsc{MC}) consists of 10000 local robustness properties applied to 2 networks trained on the CIFAR10 dataset~\cite{krizhevsky-hinton::2009:CIFAR}.
One of these networks is a convolutional network and the other is a residual network.
We could not find an available version of the benchmark used by \mipverify to evaluate its original input format.
This benchmark is supported by \tool for verifiers that can support residual connections, including: \planet, \deepzono, \deeppoly, \refinezono, and \refinepoly.
While the benchmark is supported by the version of \mipverify used in its study, it is not supported through \tool, since the publicly available version of \mipverify does not support residual connections.

The \neurify MNIST (\textsc{NM}) benchmark~\cite{wang-etal:NIPS:2018:neurify} consists of 500 $L_\infty$ local robustness properties across 4 MNIST networks, 3 fully connected networks with 58, 110, and 1034 neurons respectively, and a convolutional network with 4814 neurons.
The original benchmark was provided in \neurify-NNET format, with properties hard-coded into the verifier.
\tool enables almost all verifiers to run on this benchmark.
\reluplex cannot be run due to the presence of convolutional layers, which are not supported.
\mipverify cannot be run due to the presence of non-hypercube input constraints. While this limitation of the verifier can be satisfied with a property reduction for fully-connected networks, \tool does not currenly support such a reduction for convolutional networks.

The \neurify Drebin (\textsc{NDb}) benchmark~\cite{wang-etal:NIPS:2018:neurify} consists of 500 $L_{\infty}$ local robustness properties across 3 fully connected Drebin~\cite{arp-etal:NDSS:2014:DREBIN} networks with 102, 212, and 402 neurons each.
The original benchmark was provided in \neurify-NNET format, with properties hard-coded into the verifier.
This benchmark is fully supported by \tool for all verifiers.

The \neurify DAVE (\textsc{NDv}) benchmark~\cite{wang-etal:NIPS:2018:neurify} consists of 200 local reachability properties, with 4 different types of input constraints (50 properties of each type).
The first type of input constraint is an $L_{\infty}$ constraint, which is equivalent to a hyper-rectangle constraint.
The second type of input constraint is an $L_1$ constraint, which can be written as a halfspace polytope constraint.
The third and fourth type of input constraint are image brightning and contrast, which can be written as halfspace polytope constraints.
The properties are applied to a convolutional network for an autonomous vehicle, with 10276 neurons.
The original benchmark was provided in \neurify-NNET format, with properties hard-coded into the verifier.
Similar to the \neurify MNIST benchmark, \tool enables almost all verifiers to run on this benchmark.
\reluplex cannot be run, due to the presence of convolutional layers, which are not supported, and \mipverify cannot be run due to the presence of non-hypercube input constraints.

The \deepzono MNIST (\textsc{DZM}) benchmark~\cite{singh-etal:NIPS:2018:deepzono} consists of 1700 local robustness properties, subsets of which are applied to 10 networks trained on the MNIST dataset.
The networks have varied structures and activation functions: 3 networks are fully connected, 1 of which uses ReLU activations, 1 with Tanh activations, and 1 with Sigmoid activations; 6 are convolutional, 4 of which have ReLU activations, 1 with Tanh activations, and 1 with Sigmoid activations; and 1 is a residual network.
The networks in the original benchmark were provided in a custom human-readable text format, with properties hard-coded into the verifier.
\tool does not increase the support for this benchmark due to the presence of both a residual network and non-ReLU activation functions.

The \deepzono CIFAR10 (\textsc{DZC}) benchmark~\cite{singh-etal:NIPS:2018:deepzono} consists of 1700 local robustness properties, subsets of which are applied to 5 networks trained on the CIFAR10 dataset.
The networks have varied structures and activation functions: 3 networks are fully connected, 1 of which uses ReLU activations, 1 with Tanh activations, and 1 with Sigmoid activations; and 2 are convolutional with ReLU activations.
The networks in the original benchmark were provided in a custom human-readable text format, with properties hard-coded into the verifier.
\tool enables \verinet to run on this benchmark.
Other verifiers are not supported due to the non-ReLU activation functions.

The \deeppoly MNIST (\textsc{DPM}) benchmark~\cite{singh-etal:POPL:2019:deeppoly} consists of 1500 local robustness properties, subsets of which are applied to 8 networks trained on the MNIST dataset.
The networks have varied structures and activation functions: 5 networks are fully connected, 3 of which uses ReLU activations, 1 with Tanh activations, and 1 with Sigmoid activations; and 3 are convolutional with ReLU activations.
The networks in the original benchmark were provided in a custom human-readable text format, with properties hard-coded into the verifier.
\tool enables \verinet to run on this benchmark.
Other verifiers are not supported due to the non-ReLU activation functions.

The \deeppoly CIFAR10 (\textsc{DPC}) benchmark~\cite{singh-etal:POPL:2019:deeppoly} consists of 800 local robustness properties, subsets of which are applied to 5 networks trained on the CIFAR10 dataset.
The networks have varied structures: 3 networks are fully connected with ReLU activations; and 2 are convolutional with ReLU activations.
The networks in the original benchmark were provided in a custom human-readable text format, with properties hard-coded into the verifier.
\tool enables several additional verifiers to support this benchmark.
In particular, it enables most verifiers that can be applied to convolutional networks with relu activations.

The \refinezono MNIST (\textsc{RZM}) benchmark~\cite{singh-etal:ICLR:2019:refinezono} consists of 800 local robustness properties, subsets of which are applied to 8 networks trained on the MNIST dataset.
5 networks are fully connected with ReLU activations and 3 are convolutional with ReLU activations.
The networks in the original benchmark were provided in a custom human-readable text format, with properties hard-coded into the verifier.
\tool enables several additional verifiers to support this benchmark.
In particular, it enables most verifiers that can be applied to convolutional networks with relu activations.

The \refinezono CIFAR10 (\textsc{RZC}) benchmark~\cite{singh-etal:ICLR:2019:refinezono} consists of 200 local robustness properties, subsets of which are applied to 2 networks trained on the CIFAR10 dataset.
One of the networks is fully connected with ReLU activations and the other is convolutional with ReLU activations.
The networks in the original benchmark were provided in a custom human-readable text format, with properties hard-coded into the verifier.
\tool enables several additional verifiers to support this benchmark.
In particular, it enables most verifiers that can be applied to convolutional networks with relu activations.

The \refinepoly MNIST (\textsc{RPM}) benchmark~\cite{singh-etal:NeurIPS:2019:refinepoly} consists of 600 local robustness properties, subsets of which are applied to 6 networks trained on the MNIST dataset.
4 networks are fully connected with ReLU activations and 2 are convolutional with ReLU activations.
The networks in the original benchmark were provided in a custom human-readable text format, with properties hard-coded into the verifier.
\tool enables several additional verifiers to support this benchmark.
In particular, it enables most verifiers that can be applied to convolutional networks with relu activations.

The \refinepoly CIFAR10 (\textsc{RPC}) benchmark~\cite{singh-etal:NeurIPS:2019:refinepoly} consists of 300 local robustness properties, subsets of which are applied to 3 networks trained on the MNIST dataset.
Two of the networks are convolutional with ReLU activations and the third is a residual network with ReLU activations.
The networks in the original benchmark were provided in a custom human-readable text format, with properties hard-coded into the verifier.
\tool enables the \planet verifier to support this benchmark.
In particular, it enables most verifiers that can be applied to convolutional networks with relu activations.
Other verifiers do not support the residual structure of one of the networks.

The \verinet CIFAR10 (\textsc{VC}) benchmark~\cite{henriksen-lomuscio:ECAI:2020:verinet} consists of 250 local robustness properties applied to 1 convolutional network with ReLU activations.
The networks were provided in \onnx format, with hard-coded properties.
\tool enables support of this benchmark by most of the integrated verifiers.
\reluplex does not support convolutional networks, and \mipverify does not support properties with input constraints that are not hyper-cubes.

\end{document}